\documentclass[12pt]{article}

\usepackage{tabularx}
\usepackage{amsmath, amsthm, amssymb}
\usepackage{algorithm}
\usepackage{algorithmic}
\usepackage{enumerate}
\usepackage{nccmath}
\usepackage[margin=3cm]{geometry}
\usepackage{bm}
\usepackage{multirow}
\usepackage{hhline}
\usepackage{xurl}
\usepackage{mathtools}

\newcommand{\ph}{\varphi}

\newcommand{\bmath}{\boldsymbol}
\newcommand{\bb}{\mathbb}

\usepackage{enumerate}

\usepackage{array}
\newcolumntype{L}{>{\arraybackslash}m{7cm}}
\usepackage{tcolorbox}
\newcommand{\fon}[1]{\fontfamily{#1}\selectfont} 

\usepackage{wrapfig}

\newtheorem{statement}{Statement}
\newtheorem{definition}{Definition}
\newtheorem{theorem}{Theorem}
\newtheorem{lemma}{Lemma}
\usepackage{authblk}
\begin{document}
	\title{Practical machine learning is learning on small samples}
	
	\author{Marina Sapir} 
	\date{}
\maketitle

\begin{abstract}
	
 Based on limited  observations, machine learning discerns a dependence which is expected to hold in the future. What makes it possible? Statistical learning theory imagines  indefinitely increasing training sample to justify its approach.   In reality,  there is no infinite time or even infinite general population for learning.  
 
 Here I argue  that practical machine learning is based on an implicit assumption that underlying dependence is relatively ``smooth" :  likely, there are no abrupt differences in feedback between cases with close data points.
	
	From this point of view learning shall involve  selection of the  hypothesis ``smoothly" approximating the training set.   I formalize this as  Practical  learning paradigm. The paradigm includes terminology and rules for description of learners.  Popular learners (local smoothing, k-NN, decision trees, Naive Bayes,    SVM for classification and for regression) are shown here to be implementations of this paradigm. 

\end{abstract}

\section{Givens, goals and assumptions}

Denote $\Omega$ the set of real life objects of interest.  For example, this  may be patients with skin cancer, or bank clients or  engine failures. 
The objects have two types of properties: hidden (feedback) and manifested, which are called  \textbf{features}.  The properties are assigned  numerical values.  Denote $X$ domain of  numerical feature vectors for objects in $\Omega$.   The hidden property has values in a numerical domain denoted by $Y$. 

The \textbf{underlying  dependence} $\varphi:   X \rightarrow Y$ is what we learn. It is  expected to be inexact.  There are several reasons for inexactness, for example:
\begin{itemize}
	\item   we do not know what the feedback, actually, depends on; 
	\item  numerical evaluations of features have intrinsic  uncertainty;
	\item the objects in $\Omega$ and the dependence $\varphi$ are changing in time; new objects may somehow be different from what we observed before.  
\end{itemize}

\textbf{Hypothesis} is a function $f: X \rightarrow Y.$
\textbf{Cases} are pairs $\langle x, y \rangle$, where $x \in X, y \in Y.$ If $f$ is a hypothesis, then cases $\langle x, f(x)\rangle$ are called \textbf{hypothetical.} Denote $\ulcorner f \urcorner$ all hypothetical cases for the hypothesis $f.$ 
\textbf{Observations}  are cases from direct inspections of objects in $\Omega.$ A \textbf{training set} $T$ is a finite set o observations available for analysis.   By default, no two cases in $T$ have identical first component. 

Denone $\mathbb{M}(h, T) = T \cup \ulcorner h \urcorner$ set of all cases for the hypothesis $h$ and the training set $T:$ all these cases together are representation of the underlying dependence. 

 The next assumptions summarizes our intuitive understanding of the problem:

\begin{tcolorbox} [title =Implicite learning assumptions, fonttitle=\fon{pbk}\bfseries]
\begin{itemize}
	\item 	Any two observations $\langle x_1, y_1 \rangle,  \langle x_2, y_2 \rangle $ with ``similar" features $x_1, x_2$ likely have ``similar" feedback $y_1, y_2.$
	\item  For any two hypotheses $h_1, h_2,$ a hypothesis $h \in \{h_1, h_2\}$  with lower ``inconsistency"  on the set of cases $\mathbb{M}(h, T)$ represents the underlying dependence $\varphi$  better. 
	\item The goal of learning is to find a hypothesis with lowest ``inconsistency" on $\mathbb{M}(h, T).$
\end{itemize} 	
	\end{tcolorbox}
	
	The assumption explain why we undertake the learning without having an indefinitely increasing training set or a knowledge of all relevant features.  I will show that the assumption is  sufficient to build efficient learners in many cases. 

Below,  I define the concepts  of these assumptions in general terms and  show on examples of popular learners how they are interpreted. 

I will overload the  notation $[\cdot]$ to mean different things for different types of data. For: 

\hfill \break
\begin{tabular}{lll}
Vector $x:$ & $ \big[x \big]_i$ & is  $i$-th component of $x$; \\	
Case $\alpha:$ & $ \big[\alpha \big]_i$ & is $i$-th component of $\alpha$;\\
Sequence of cases $T:$ & $ \big[T \big]_i$ & is sequence of  $i$-th components of cases in $T$.\\
\end{tabular}
\hfill \break

 \section{Traditional  views on ML} 
 
 It is interesting to observe that ML problem does not have a general definition. Here are some popular views on the problem.

\subsection{Is ML an induction?}

It is a common belief that ML is an induction inference. 

The dictionary says that induction is a method of reasoning from a part to a whole, from particulars to generals, or from  an individual to universal. 

Roughly,  induction extends common known property of observations onto the general population. 
 
 Say, we know  that for any observation $\alpha$ in a training set  $T$  we have $ \big[\big[\alpha \big]_1\big]_i \approx 1,$ so we infer that for any case $\beta$  in the general population $\big[\big[\beta \big]_1 \big]_i \approx 1.$ This act of projecting a known property  of the training set onto the general population would be called induction. 

Learning requires to select a hypothesis to characterize the dependence between features and feedback on the training set $T$.    So, in ML,  the hypothesis was not known, and the act of learning does not include extending it on the general population. Therefore, learning as a logical procedure  does not have any properties of induction. 

Peirce called selection of a hypothesis which optimally agrees with observations  an\textbf{ abduction procedure} \cite{PeirceV1}. 

After we learned the hypothesis on the training data, we implicitly apply the induction step to extrapolate  it  onto the general population.  Then we use a deduction as well to obtain predicted feedback for some new observations.  Thus, both induction and deduction follow learning, but are not the learning itself. 

\subsection{Is ML a prediction problem?} 

ML is sometimes understood as a prediction problem. It is assumed that the goal is to predict values of the inexact dependence $\ph$ on  manifested features of some future observations.

  For example, here is how the prediction problem is formulated in \cite{Stability}: 
  \begin{quote}
  Given a training set $S$ and data point $x$  of a new observation  $\langle x,  ?  \rangle$; predict its feedback $y.$
  \end{quote} 
  The main issue with this definition is that it does not provide a criterion  to evaluate the decision.  It appears that to select between the hypotheses we need to know what is not given: the future. 
  
  Therefore, the ML problem can not be formulated as a prediction problem, as it is defined in \cite{Stability}.

\subsection{Does Statistical Learning describe ML? }

Statistical Learning  (SL) theory is  the only commonly accepted theoretical approach to ML.

	This is how the  proponents of the SL theory   understand the problem: 
\begin{quote}
	Intuitively, it seems reasonable to request that a learning algorithm, when presented more and more training examples, should eventually  ``converge” to an optimal solution. \cite{StatTheory} 
\end{quote}

The ``optimal solution'' here means  a hypothesis to which the infinitely increasing training set converges. 
		
 But why is it intuitive to assume that there is ever more training examples, if  in reality  it is not the case?   V.  Vapnik \cite{VapnikBook} formulated the justification of the statistical approach  in the most direct way 

\begin{quote}
	\textit{Why do we need an asymptotic theory $\langle \cdots \rangle$ if the goal is to construct algorithms from a limited number of observations?
		The answer is as follows:
		To construct any theory one has to use some concepts in terms of which the theory is developed $\langle \cdots \rangle.$ }
\end{quote}

In other words, statistical learning theory  assumes  indefinite increase of the training set so that the statistical approach can prove some results. Statistics  has laws of large numbers, so the problem has to be about ever increasing training sets and  convergence. 

From the statistical learning theory point of view, empirical risk minimization (ERM), or fitting the observations, shall provide a satisfactory solution. 

Since the assumptions are not satisfied, the results of the theory can not be useful.  Indeed, ERM is rejected by practitioners, because it is known to ``overfit".  The same is true for some other learners, satisfactory from statistical learning perspective.

We have to conclude that statistical learning theory  does not describe ML as a it is known to practitioners.

\section{Practical learning. Main concepts}

\subsection{Definitions}
\begin{definition}
	A \textbf{Problem statement} is the tuple $\{X, Y, F, v\},$
	where $X, Y$ are domains, $F$ is class of functions $X \rightarrow Y$, $v$ is an optional  vector of parameters.
	\end{definition}

\begin{definition}
\textbf{Practical learning paradigm}  is  the tuple

$$ \Xi = \langle \mathcal{P}, T, H, \mathit{\xi},\mathit{\mu}, \Lambda \rangle,$$
where\\

\begin{align*}
	\mathcal{P} : & \text{ problem statement} \{X, Y, F, v\}\\
	T: & \text{ independent variable, training set}\\ 
	H(f, T): &  \;  \text{ \textbf{baseline cases} for } f \in F, \\
	&  H(f, T) \subseteq Q_1: \; Q_1 \in \{T, \ulcorner f \urcorner\}\\
	\xi(\alpha, f): &  \text{ \textbf{counterparts} for } \alpha, \alpha \in H(f, T):  \\
	&  \xi(\alpha, f) \subseteq Q_2: Q_2 \in \{T, \ulcorner f \urcorner\} \setminus Q_1\\
   \mu(\alpha, v): &  \; H(f, T) \rightarrow \mathbb{R}^+, \text{ \textbf{case- inconsistency} between } \alpha \text{ and } \xi(\alpha, f)  \\
	\Lambda(f, T, v): &  \text{  \textbf{total inconsistency}, } \\ 
	 & \forall \alpha \in H(f, T),  \Lambda(f, T, v)\text{  monotonously depends on } \mu(\alpha, v)  
\end{align*}
\end{definition}

\begin{definition}
A \textbf{practical learner} includes  
\begin{itemize}
\item interpretation  of all terms in the tuple $\Xi$  using  formulas and algorithms;  
\item  procedure, which, given the problem statement $\mathcal{P} = \{X, Y, F, v\}$ and a training set $T,$ searches  the solution 
$$f = arg \min_{h \in F_0}\Lambda(h, T, v) .$$ 
\end{itemize}
\end{definition}

\subsection{Intended meaning}

The practical learning paradigm represents a terminology and rules for description of machine learners.  One may notice here some rough approximation to a formal system and a language of a  learning logic, which may be developed in a future. 

The fundamental concepts for each learner are \textit{baseline cases} and their \textit{counterparts}. Baseline cases and counterparts belong to different types of cases (if baseline cases are observations, then counterparts are hypothetical cases and vice versa). Counterparts  $\xi(\alpha, f, T)$ are supposed to be defined as  ``\textit{similar}"  to $\alpha$ both in features and feedback.   The \textit{case inconsistency} $\mu(\alpha, v)$  evaluates the degree by which this assumption of similarity is violated.  \textit{Total inconsistency} evaluates degree to which the assumption of similarity  is violated on the set of all baseline cases.  Thus, a learner described within the Practical learning paradigm  evaluates the violation of assumed ``closeness" between observations and the hypothetical cases for a given hypothesis. It means, a learner  interprets  the Implicit learning assumptions in its own way. 

For example,  in empirical risk minimization (ERM) strategy,  the \textit{baseline cases} are observations $H(f, T)  =T.$  And for a baseline case  $\alpha = \langle x_i, y_i \rangle, i \in [1, m],$ its set of \textit{counterparts}  $\xi(\alpha, f, T)$ consists of a single hypothetical  case of a hypothesis $f$
$$\xi(\alpha, f, T) = \{ \langle x_i,  f(x_i)\rangle  \}.$$
The\textit{ case - inconsistency} is defined as 
$$\mu(\alpha) = \|y_i - f(x_i)\|$$
and\textit{ total  inconsistency} as 
$$\Lambda(h, T) = \sum_{\alpha \in H(h, T)} \mu(\alpha).$$

The ERM learner is the most primitive example of a practical learner, because here each set of counterparts contains only one case and the inconsistency between a baseline case and its counterpart is defined by the distance between feedback of the cases.  

Next I will show that many popular learning algorithms are, indeed, practical  learners.

\section{Local smoothing is a practical learner}

The problem statement here $\mathcal{P} = \{X, Y, F, x_0\}$ has an additional parameter $x_0 \in X$ and, may be, some other parameters. The task is to ``smoothly" approximate  a function $f \in F$ in a point $x_0$ given estimates of the function values in some points, the training set $T = \{\langle x_i, y_i \rangle, i = 1, \ldots m\} .$  

The hypotheses class $F$ consists  of functions defined on a single point $x_0$. For each hypothesis $h \in F$ there is a single baseline case, the hypothetical case  $\alpha(h)  =\langle x_0, h(x_0) \rangle.$

The\textit{ counterparts} $\xi(\alpha(h) )$ for a \textit{baseline case} $\alpha(h)$   are defined as subset of the training set $T$ such that the points $\big[ \xi(\alpha(h)) \big]_1$ are in some sense closest to the point $x_0.$ A smoothing learner determines how exactly these counterparts are selected, perhaps, using some additional parameters in the problem statement. 

The\textit{case-inconsistency} may be defined by the rule: 
$$\mu(\alpha(h)) = \Big| h(x_0)  - \mathit{mean}\big( \big[\xi(\alpha(h), T) \big]_2\big)\Big|.$$

Since there is only one baseline case,  \textit{total inconsistency} is the same as single case - inconsistency, $\Lambda(h, T) = \mu(\alpha(h)).$

A smoothing learner searches for a hypothesis  minimizing  total inconsistency. Obviously, the hypothesis 
$$h(x_0) = \mathit{mean}\big( \big[\xi(\alpha(h), T) \big]_2\big)$$ solves the problem. 

Smoothing assumes that the all counterparts in  $\xi(\alpha(h))$  together better represent expected value of the function than the single closest data point, for example,  because of inherent measurement uncertainty  in $x, y.$ 

Even though smoothing is not typically considered among the machine learning problems, it is an archetypal problem of   practical machine learning. A practical learner, whichever problem it solves,   does ``smoothing" of a sort, finding a less ``inconsistent"  approximation of the inexact underlying dependence.

\section{k-NN is a practical learner}

The problem statement  $\mathcal{P}$ is formulated here with $X$ being a metric space and binary feedback in $Y = \{0, 1\}$. It also has the additional parameters $x_0 \in X,$ $k \in \mathbb{N}.$ The class of hypotheses here consists of two functions 
$$F= \{h_0(x_0) = 0, h_1(x_0) = 1\}, $$ 
and each function $h \in F$ has a single hypothetical case,  which is taken by the  learner as a  baseline case $H(h) = \{\alpha(h) \} = \{ \langle x_0, h(x_0)\rangle\}.$ 

 For the single baseline case $\alpha(h)$ the counterparts $\xi(\alpha(h), T)$ are $k$ observations  from the training set $T$ with  feature vectors closest to the $x_0.$ 
 
The case - inconsistency $\mu(\alpha(h))$ is defined by the same formula, as for smoothing: 
$$\mu(\alpha(h)) = \Big|  h(x_0) - \mathit{mean}\big( \big[\xi(\alpha(h)) \big]_2\big)\Big|.$$

The same goes for total inconsistency  $\Lambda(h, T) = \mu(\alpha(h)).$ 

The learner selects as a solution $\mathbf{f}$ one of two hypotheses in $F$ which minimizes total inconsistency $\Lambda(h, T).$

\section{Decision tree is a practical learner}

In the problem statement of this learner the domain $X$ has only ``ordinal'' features: every  feature has finite number of ordered  values; there are no operations on feature values.  The feedback of observations is binary, $Y = \{0, 1\}$.  There is also an additional parameter  $x_0 \in X$ and some other parameters which control the construction of the tree.  The class of functions $F$ contains two hypothesis $h_0(x_0) = 0, h_1(x_0) = 1..$ 

The hypothetical case $\alpha(h) = \langle x_0, h(x_0)\rangle$ for the hypothesis $h$ is the only baseline case for this hypothesis. The decision tree learner defines the set of counterparts  $\xi(\alpha(h)))$ not by a formula, but by a certain procedure partitioning the domain $X$ into the subdomains, defined by intervals of features values.

To determine the subdomains,  the learner starts with whole domain, splits it in two subdomains by a value of some feature.  There are certain criteria for stopping the splitting of a given subdomain and declaring it a ``leaf".  Various versions of the  learner have different criteria. After a leaf was found, the algorithm switches to the next subdomain which did not undergone all possible splits yet. 
 
The set of counterparts $\xi(\alpha(h))$ is defined by the leaf $A(x_0),$ where $x_0$ belongs: the learner picks as an counterpart every observation, which has its first component in $A(x_0):$ $$ \beta \in  \xi(\alpha(h)) \Leftrightarrow \big( \beta \in  T \; \& \;[\beta]_1 \in A(x_0)\big).$$

 The case - inconsistency between $\alpha(h)$ and $\xi(\alpha(h))$ is 
 $$ \mu(\alpha(h)) = \Big| h(x_0) - \mathit{mean}\big[ \xi(\alpha(h))\big]_2\Big|,$$
 the same as for $k$-NN and smoothing.  Total inconsistency for a single baseline case coincides with case - inconsistency:
 $\Lambda(T, h) = \mu(\alpha(h))$.  And the learner picks the hypothesis with the lowest total inconsistency, indeed.

\section{Naive Bayes is a practical learner}

 In the problem statement $\mathcal{P}$, the domain  $X$ has  $n$-dimensional vectors of nominal values. Denote $C^i$ finite  set of values of a feature $i$ in $X$ and assume for any $i \neq j$ $C^i \cap C^j = \emptyset.$ $Y = \{0, 1\}.$
  
  The learner designed for the situation when  every feature is related with the feedback, and features compliment each other: perhaps, they were selected (designed) this way.  The problem statement has an additional parameter $x_0 \in X.$

The learner re-defines the problem statement $\mathcal{P}$  as new problem statement  $\mathcal{P}^\prime = \{X^\prime, Y, F^\prime, x_0^\prime\}. $ The domain 
$$X^\prime = \bmath{C} = \cup C^i.$$ 
Instead of the vector $x_0$, the problem statement $\mathcal{P}^\prime$ contains sequence $x_0^\prime = \{ [x_0]_1, \ldots, [x_0]_n\}$ of its coordinates.  The class of functions $F^\prime$ contains 2 constant functions with values in  $Y.$ Each function $h \in F^\prime$ is defined in $n$ points of the sequence  $x_0^\prime.$

Accordingly, the training set $T$ for the problem $\mathcal{P}$ is transformed into training set $T^\prime,$
where 
every observation $\langle x, y \rangle \in T$  is transformed into $n$ observations $$\{  \langle \big[x]_1, y \rangle, \ldots, \langle  \big[x]_n, y \rangle\}$$ in the new training set $T^\prime.$

Each function $h \in F^\prime$ has $n$ hypothetical cases which are taken as baseline cases $H(h).$

For each baseline case $\alpha  \in H(h),$ its counterparts are defined by the formula
$$\xi(\alpha, T^\prime) = \{ \beta:  (\beta \in T^\prime) \& ( \big[\beta \big]_1 =  \big[\alpha \big]_1)\}.$$

Then the local inconsistency for a baseline case $\alpha \in H(h)$ and its counterparts  is defined as 
$$\mu(\alpha) =  \frac{1}{ \|\xi(\alpha, T^\prime)\| } \underset{\beta \in \xi(\alpha, T^\prime)}{\sum} I\Big( \big[\beta \big]_2 \neq  \big[\alpha \big]_2\Big).$$

The total inconsistency for a hypothesis $h$ with baseline cases  $H(h)$ is defined by formula
$$\Lambda(h, T) = \underset{\alpha \in H(h)}{\prod} \mu(\alpha).$$ It is obvious that $\Lambda(h, T)$ is monotonous by each $\mu(\alpha)$ from $H(h).$

Thus,  Naive Bayes is a practical learner, and it does not need probabilistic reasoning for its justification.

\section{Linear SVM for classification is a practical learner}

The problem statement contains domains   $X = \mathbb{R}^n, Y = \{-1, 1\}$ and an additional parameter $w.$ The class of hypotheses $F$ consists of  linear functions $f(x) $ with $n$ variables.  For a function $f(x) = x^T b+ a$ denote $$\big[f\big]_1 = b, \big[f\big]_2 = a.$$ 

Denote $m$ the number of elements in the training set $T = \{\langle x_1, y_1 \rangle, \ldots , \langle x_m, y_m \rangle\}.$

We will need a fact from linear algebra: the distance $\rho(x_0, f)$  from a point $x_0$ to the hyperplane $f(x) =  0$ is defined by formula \cite{Shalev} $$\rho(x_0, f) = |f(x_0)|.$$ 

The problem is formulated explicitly as minimization of the criterion with some slack variables  $\bmath{\zeta} = \{\zeta_i, i \in [1,m]\}$
\begin{tcolorbox} [title = Linear SVM, fonttitle=\fon{pbk}\bfseries]
	\begin{align}
L(f, T, \bmath{\zeta})  =   w \, \| \big[f\big]_1 \|^2 + \frac{1}{m} \sum_{i =1 }^m \zeta_i  	\label{cri}	 \\ 
\text{s.t. }  \underset{i \in [1:m]}{\forall \; i}  \; \;  y_i  f(x_i) \ge 1 - \zeta_i\; \text{ and } \; \zeta_i  \ge 0. 
		\label{conditions} 
	\end{align}
\end{tcolorbox}

I need to show that the learner, actually, minimizes what can be called total inconsistency between properly defined  baseline cases and their counterparts in accordance with the Practical learning paradigm. 

For this, I want to simplify the criterion. Let us consider  an alternative problem:

\begin{tcolorbox} [title = Linear SVM*, fonttitle=\fon{pbk}\bfseries]
\begin{align} 
	L^*(f, T)  =   w \, \| \big[f\big]_1 \|^2 + \frac{1}{m} \sum_{i =1 }^m \zeta_i  	\label{cri1}\\
	\text{s.t. } \forall i
	\begin{cases}
		\zeta_i = 0, & \text{if  } y_i f(x_i) \ge 1 \\
	    \zeta_i = 1 - y_i f(x_i), & \text{if  } 	y_i f(x_i) < 1 
	    	\label{conditions1} 
	\end{cases}
\end{align}
\end{tcolorbox}

The slack variables $\bmath{\zeta} = \{\zeta_i,  \in [1, m]\}$ are restricted by inequalities (\ref{conditions})  in SVM problem, but they determined uniquely for each function  $f \in F$ by equalities (\ref{conditions1}) in SVM*.   This is why the criterion (\ref{cri}) of the SVM problem  depends of the slack variables $\bmath{\zeta}$ and the criterion in SVM* (\ref{cri1}) does not depend on $\bmath{\zeta}.$

\begin{statement}
	For a function $f \in F,$ if slack variable $\zeta_i,  1 \le i \le m,$ satisfies (\ref{conditions1}) it  also satisfies  (\ref{conditions}). 
\end{statement}
\begin{proof}
	Suppose,  the conditions (\ref{conditions1}) are true for $\zeta_i$. 
	
	First, let us assume $y_i \, f(x_i) \ge 1.$ Then $\zeta_i =0$ according with (\ref{conditions1}). We need to show that in this case both inequalities  (\ref{conditions}) 
	\begin{align}
		y_i \, f(x_i) \ge & \,1 - \zeta_i \label{thefirst} \\ 
		 \zeta_i \ge & \, 0. \label{thesecond}
\end{align}
	are satisfied.
	
	Substituting 0 for $\zeta_i$ in (\ref{thefirst}) we will get what we assumed to be true in this case ($y_i \, f(x_i) \ge 1$). Substituting 0 in (\ref{thesecond})  will get trivial formula $0 \ge 0$. 
	
	Consider the second case in (\ref{conditions1}):   $y_i\, f(x_i) < 1, $  with $\zeta_i = 1 - y_i \, f(x_i).$ Substituting this expression instead of $\zeta_i$ in (\ref{thefirst}) 
	we will get $$y_i \, f(x_i) \ge 1 - (1 - y_i \, f(x_i))$$ which leads to trivial inequality $y_i \, f(x_i) \ge y_i \, f(x_i).$
	
	Substituting expression for $\zeta_i$ in (\ref{thesecond}) we will get $$1 - y_i \, f(x_i) \ge 0$$ or $$y_i \, f(x_i) \le 1,$$ which is true, because we consider the case $y_i\, f(x_i) < 1.$
\end{proof}

\begin{statement} For any $f \in F.$ for any $i \in [1, m]$, if slack variable  $\zeta_i$ satisfy conditions (\ref{conditions}), and slack variable $\zeta^*$ satisfy conditions (\ref{conditions1}), then  $\zeta^*_i \le \zeta_i.$ 
\end{statement}
\begin{proof} 
	Lets consider the options in (\ref{conditions1}) again and show that for each of two cases  $\zeta^*_i \le \zeta_i.$ 
	
	If $y_i f(x_i) \ge 1, $ then $\zeta^*_i = 0.$ By  (\ref{conditions}) $\zeta_i \ge 0,$ therefore $\zeta_i \ge \zeta^*_i$ in this case.

	If $y_i f(x_i) < 1$, then $\zeta^*_i = 1 - y_i f(x_i)$ from  (\ref{conditions1}).
	From  (\ref{conditions}), $$\zeta_i \ge 1 - y_i f(x_i) = \zeta^*_i.$$
\end{proof}

\begin{lemma}
	The problems  Linear SVM* and   Linear SVM are equivalent.
\end{lemma}

\begin{proof}
	
	Let us show that $f^*$ with slack variables $\bm{\zeta}^*(f)$ minimizing $L^*(f, T)$ minimize $L(f, T, \bm{\zeta(f)})$ as well.
	
	The criterion $L(f, T, \bm{\zeta}(f))$ in (\ref{cri}) may have different values for the same function $f$  depending on the slack variables $\bm{\zeta}(f)$ satisfying (\ref{conditions}). By the Statement 1, $\bm{\zeta}^*(f)$ satisfies  (\ref{conditions}). Therefore, $L(f, T, \bm{\zeta}^*(f))$ is one of the values this criterion may take for the function $f$. 
	
	Also, it is easy to see that $$L(f, T, \bm{\zeta}^*(f)) = L^*(f, T)$$ as it is defined in (\ref{cri1}). 
	
	By the Statement 2, for any $i$ $\zeta_i^*(f) \le \zeta_i(f).$ Therefore, 
	 $$L(f, T, \bm{\zeta}(f)) \ge L(f, T, \bm{\zeta}^*(f)).$$
	 
	By definition,  $f^*$ is solution of SVM*, therefore for any $f \in F$
	$$L^*(f, T) \ge L^*(f^*, T).$$ 
	
	Now we have for any $f \in F$ for any $\bm{\zeta}(f)$ satisfying (\ref{conditions})
	\begin{align*}
		L(f, T, \bm{\zeta}(f)) & \ge L(f, T, \bm{\zeta}^*(f))\\
	 & =  L^*(f, T)\\
	 & \ge L^*(f^*, T) \\
	 & = L(f^*, T, \bm{\zeta}^*(f^*)).
	 \end{align*}

Since $\bm{\zeta}^*(f^*)$ satisfies (\ref{conditions}) by the Statement 1, this proves that $f^*$ with slack variables defined by (\ref{conditions1}) provides minimum to the criterion in the problem SVM. 

Now, suppose  $f^\prime$ with slack variables $\bm{\zeta}(f^\prime)$ minimize $L(f, T, \bm{\zeta}(f)).$
Let us show that in this case  $f^\prime$ minimizes $L^*(f, T)$ as well.

By the Statement 2, for every $f \in F$, for every  $i$, if $\zeta_i(f)$ satisfies restrictions (\ref{conditions}) and $\zeta_i^*(f)$ satisfies (\ref{conditions1}) then $$\zeta_i^*(f) \le \zeta_i(f).$$  By the Statement 1, in this case $\bm{\zeta}^*$ satisfies restrictions (\ref{conditions}) as well. Therefore, $$L(f^\prime, T, \bm{\zeta}^*(f^\prime))$$ is a valid value of the criterion of the SVM problem and the criterion achieves minimum when $\bm{\zeta}(f^\prime) = \bm{\zeta}^*(f^\prime).$

It means that for any $f \in F$ 
$$L(f, T, \bm{\zeta}(f)) \ge L(f^\prime, T, \bm{\zeta}^*(f^\prime))$$
when $\bm{\zeta}(f)$ satisfies (\ref{conditions}). In particularly, it is true when $\bm{\zeta}(f) = \bm{\zeta}^*(f).$
So for any $f \in F$
$$L(f, T, \bm{\zeta}^*(f)) \ge L(f^\prime, T, \bm{\zeta}^*(f^\prime)).$$
At the same time for any $f \in F$ $$L(f, T, \bm{\zeta}^*(f)) = L^*(f, T).$$ It follows that for any $f \in F$
$$L^*(f, T) \ge L^*(f^\prime, T).$$
This proves that $f^\prime$ solves the the SVM* problem. 

\end{proof}

\begin{theorem} A learner solving the problem Linear SVM is a practical learner.
	\end{theorem}
\begin{proof}
By the Lemma 1, the problems Linear SVM and Linear SVM* are equivalent. So it is enough to prove the theorem for Linear SVM*.  

For each hypothesis $f \in F, $ the learner de facto takes training set $T$  as \textit{baseline cases} $$H(f, T) = T = \{ \alpha_i = \langle x_i, y_i \rangle,\, i\in [1,m]\}. $$

 There are two half-spaces of interest in $X:$ 
 $$z(f, y) = \{x: \; y \, f(x) \ge 1\}, y \in \{-1, 1\}.$$
 
 The set $z(f, y)$ is one of two  half-spaces separated by a hyperplane $y \, f(x) - 1 = 0.$

For a baseline case  $\alpha_i  = \langle x_i, y_i\rangle \in T$ its set of \textit{counterparts}   $\xi(\alpha_i, f)$  is determined by $y_i:$
$$\xi(\alpha_i, f) = \{\langle x,  y_i \rangle: \; x \in z(f, y_i) \}$$

Denote $\rho(x, z(f, y) )$ the distance from $x$ to the hyperplane, separating the half- space $z(f, x):$
$$\rho(x, z(f, y) ) = |y \, f(x)  - 1|.$$ 
If $x \not \in z(f, z)$ it means $y\, f(x) < 1.$ For this case 
\begin{align}
\rho(x, z(f, y) ) = 1 - y \, f(x).   \label{dist}
\end{align}
Let us define case - inconsistency by the rule:
\begin{align*}
	\mu(\alpha_i) = 
	\begin{cases}
		0,  & \text{ if  } \alpha_i \in \xi(\alpha_i, f)\\
		\rho(x_i, z(f, y_i))  & \text{ otherwise}
		\end{cases}
\end{align*} 

In view of (\ref{dist}) it can be rewritten as 
\begin{align*}
	\mu(\alpha_i) = 
	\begin{cases}
		0,  & \text{ if  } y_i \, f(x_i) \ge 1\\
		1 - y_i \, f(x_i) & \text{ otherwise.}
	\end{cases}
\end{align*} 

This shows that for every $i \in [1,m]$ $$\mu(\alpha_i) = \zeta_i$$ as defined in (\ref{conditions1}).

It is obvious that $L^*(f, T)$ monotonously depends on each $\mu(\alpha_i) = \zeta_i.$

It proves that the learner which solves Linear SVM* problem (and it equivalent Linear SVM problem) is a practical learner.
\end{proof}

The criterion (\ref{cri1}) has two components. The first one, $ w \, \| \big[f\big]_1 \|^2$ depends only on the hypothesis $f.$ 

This component $a(f) = \| \big[f\big]_1 \|^2$ directly evaluates derivative of the hypothesis. The smaller is $a(f)$ , the less the hypothesis changes in each point. We can say that this component  characterizes  inconsistency of the hypothetical cases with each other, while the second component of the criterion, $\sum(\zeta_i),$ characterizes inconsistency between the training set and the hypothetical instances. Both together they define the goal of learning, minimization of inconsistency on $\mathbb{M}(f, T) = T \cup \ulcorner f \urcorner,$ as is specified by  the Implicit learning assumptions. 

Specifics of this learner is that it ignores small errors: when $f(x)$ has wrong sign $f(x) y_i < 0$ but $|f(x)| < 1$  the case inconsistency is equal 0. The designer of the learner assumed all large and  only large errors are meaningful for selection of the separating hyperplane: small errors are considered a random noise. 

But when is this an optimal strategy?  If an exact  linear separation between classes exists, small errors may occur due to measurements uncertainty, but large errors signal wrong separating hyperplane. Then this approach shall work.  However,  if the optimal separation  on the general population has large errors, occurring rarely, by chance they may appear comparatively more often in a small sample. In this case, one would prefer to ignore the largest errors.

\section{Linear Support vector regression is a practical learner}

The problem statement contains domains   $X = \mathbb{R}^n, Y = \mathbb{R}.$ The class of hypotheses $F$ consists of  linear functions $f(x) $ with $n$ variables.  The problem statement includes two additional parameters $\varepsilon, \lambda.$

For the training set $T = \{\beta_1, \ldots, \beta_m\}$ the learner minimizes criterion (as defined in \cite{Hastie})

\begin{tcolorbox} [title = SV Regression, fonttitle=\fon{pbk}\bfseries]
$$L_{svr}(f, T) = \sum_{i = 1}^m V_\varepsilon\big([\beta_i]_2  - f([\beta_i]_1)\big) + \lambda \|[f]_1\|^2, $$
where 
\[
V_\epsilon(r) = 
\begin{cases}
	0, & \text{if  } |r| < \varepsilon\\
	|r| - \epsilon, & \text{otherwise}
\end{cases}
\]
\end{tcolorbox}

The \textit{baseline cases} here are observations: $H(f, T) = T.$ For a baseline case $\beta_i = \langle x_i, y_i\rangle$ its \textit{counterparts} $\xi(\beta_i, f)$ consists of the single case $\xi(\beta_i, f) = \{\langle x_i, f(x_i)\rangle\}, $ the same as in ERM method. The \textit{case-inconsistency} is also  defined similar to ERM, but   with modifier $V_\epsilon:$ 
$$\mu(\beta_i) = V_\varepsilon(|y_i - f(x_i)|).$$ The modifier allows to ignore small errors, as in SVM for classification. 

The criterion $L_{svr}(f, T)$ can be called total inconsistency because it monotonously depends on each $\mu(\beta_i).$

The second component in the criterion $L_{svr}(f, T)$ is regularization, or measure of inconsistency of the hypothetical instances, the same as in the SVM.

So, the linear support vector regression is a practical learner as well. 

\section{Advantages of Practical learning theory}

The main advantage of the theory which agrees with practical applications is that it allows to rise practically important questions which are meaningless in statistical learning  theory. Here are some of these questions.

\subsubsection{Selection of a learner}

Within statistical learning theory learners are incomparable, formulated in different languages. Besides, there is not reason to compare them: any learner which converges to the solution when the training set indefinitely increases is as good as any other such learner. 

In Practical learning the learners are described in the same language. We saw that differences between learners are explainable by assumptions about the data and the underlying dependence. This allows meaningful comparison of them. 

Besides, learners may be characterized by their robustness toward peculiarities of small samples. Question about robustness of a learner is practically important but is   meaningless in statistical learning. 

\subsubsection{Testing}
Most applied machine learning specialists know that you can not evaluate a decision rule  on the same data, on which you obtained it. But the textbooks on machine learning are usually mum about it because it does not make sense in statistical learning theory, when the data are expected to increase indefinitely and the underlying dependence is assumed to be exact. 

Practical learning allows to understand testing as a logical operation necessary in the process of learning.

\subsubsection{Outliers}
If we assume, as statistical learning does, that the training set can be increased indefinitely, and eventually will approximate the distribution and the function on data we are learning, there can not be any outliers. Otherwise it is easy to understand that training sample may contain disproportional number of observations which do not fit in the general tendency of the underlying dependence. It makes  the question about outliers become critical. 

\subsection{Data sufficiency}
Small sample may be not sufficiantly representative. This is the problem in practical applications, but it does not make sense in statistical learning theory. Practical learning allows to formulate the problem and search for its solution. 

\section{Conclusions}

Currently, theoretical machine learning is a statistical learning.    Where statistical learning sees approximation of a function with  infinitely increasing number of observations,   practical learning has to deal with uncertainty in data, non-deterministic dependencies,   limited   time-frame for learning  and all the problems which come from these restrictions. 

I propose to see machine learning not as statistical, but as a logical problem which Peirce called abduction: search for a hypothesis which ``explains" observations the best in terms of some quantifiable inconsistency of a hypothesis and observations. This allows one to approach the real life problems machine learning practitioners are facing. 
 
 Implicit learning  assumptions (ILA) are proposed here as an alternative to the assumption of indefinitely increasing training set of the statistical learning. ILA view machine learning as minimization of ``inconsistency" on the set of observations and hypothetical cases. The proposed Practical learning paradigm  includes terminology and rules to describe practical learners. The paradigm is based on the  ILA. 

 I demonstrated  that  such learners as k-NN, Naive Bayes, decision tree, linear SVM for classification, and linear SVM  for regression, which appear to have different approaches, justifications and structure,  all can be described within Practical learning paradigm.  One can say also that each of these learners is a formalization of the ILA, because it minimizes a quantifiable ``inconsistency" on the set of all the hypothetical cases and observations.  
 
 The regularization component  in the criteria of SVM and SVR learners nicely fits in the proposed view, because it evaluates inconsistency on the hypothetical cases. 

The examples shall support the conjecture  that practical learning is not  a statistical, but a logical problem. In the next articles I plan to show that  popular learners for  other machine learning problems (clustering, for example) can be explained within  the Practical learning paradigm.

\bibliographystyle{plain} 
\bibliography{Smoothing} 

\end{document}